\documentclass{article}

\usepackage{iclr2025_conference,times}

\usepackage[utf8]{inputenc}
\usepackage[T1]{fontenc}
\usepackage{amsmath,amssymb,amsthm}
\usepackage{algorithm,algorithmic}
\usepackage{booktabs}
\usepackage{float}
\usepackage{graphicx}
\usepackage{hyperref}
\hypersetup{colorlinks=true, linkcolor=red, citecolor=blue, urlcolor=magenta, bookmarks=true}
\usepackage{xcolor}
\usepackage{enumitem}

\newtheorem{theorem}{Theorem}
\newtheorem{proposition}{Proposition}
\newtheorem{definition}{Definition}

\newtheorem{remark}{Remark}

\newcommand{\R}{\mathbb{R}}
\newcommand{\F}{\mathcal{F}}

\newcommand{\norm}[1]{\left\| #1 \right\|}
\newcommand{\abs}[1]{\left| #1 \right|}
\newcommand{\E}{\mathbb{E}}

\title{Breaking the Translation Equivariance Barrier:\\
Coordinate Channels and Bounded Error Propagation\\
in Multi-Physics Neural Operator Cascades}

\author{Michael Chin\thanks{Independent Researcher. Correspondence: 40185353@qq.com.}}

\iclrfinalcopy  

\begin{document}
\maketitle
\lhead{Under review as a conference paper at ICLR 2027}

\begin{abstract}
Neural operators such as the Fourier Neural Operator (FNO) achieve remarkable accuracy in approximating solutions to partial differential equations (PDEs), but two fundamental challenges hinder their deployment in safety-critical engineering: (1)~a structural approximation barrier arising from FNO's translation equivariance that prevents learning Dirichlet boundary conditions, and (2)~the lack of uncertainty bounds when multiple operators are chained in multi-physics cascades. We address both. First, we prove that FNO's translation equivariance creates a fundamental approximation barrier for PDEs with non-trivial boundary conditions (Theorem~\ref{thm:barrier}), and that coordinate channels resolve this (Theorem~\ref{thm:coord}), validated by up to $63\times$ error reduction across five scenarios. Second, we formalize \emph{error propagation in operator cascades}: under an explicit TV-Lipschitz condition on downstream conditional mechanisms, we bound the per-stage coverage gap by the stage failure probability plus a mechanism-response term proportional to upstream prediction error (Theorem~\ref{thm:cascade}), and we show empirically that passing predicted fields between stages degrades coverage far more than error-matched perturbations. We validate on a 3D electromagnetic$\to$thermal$\to$structural cascade, where coordinate channels reduce coupling error from 18.5\% to 5.0\%. For per-stage UQ, we propose Physics-Informed Conformal Prediction (PI-CP), which embeds PDE residuals into the nonconformity score---complementary to recent physics-residual CP~\citep{gopakumar2025calibrated}---and achieves 89--91\% coverage across six physics scenarios with spatially adaptive intervals. We further identify a fundamental boundary of diffusion-based UQ: on well-conditioned forward PDEs, the trained FNO's residual is too small for diffusion models to capture meaningful uncertainty (38--53\% coverage). An industrial application---fiber-optic gyroscope bias instability prediction---validates the framework with $3.7$--$8.0\times$ improvement over analytical baselines.
\end{abstract}

\paragraph{Note on this version (v2).}\ \
This version corrects an error in an earlier preprint (v1, June 2026): Theorem~3 and its supporting Remark were stated with an input-marginal total-variation bound that is not valid (counterexamples exist with equal input marginals and arbitrarily poor downstream coverage). The theorem is restated in Theorem~3 under an explicit TV-Lipschitz condition on downstream conditional mechanisms, with a full proof; the cascade-coverage experiments are rerun with a corrected protocol that passes predicted fields downstream, alongside the original injected-perturbation results for comparison.

\section{Introduction}
\label{sec:intro}

Neural operators~\citep{li2021fno,lu2021deeponet,kovachki2023neural} learn mappings between infinite-dimensional function spaces, offering orders-of-magnitude speedup over traditional PDE solvers after training. The Fourier Neural Operator (FNO), in particular, achieves sub-1\% relative error on benchmark PDEs such as Darcy flow, Navier-Stokes equations, and elasticity. However, deploying these surrogates in safety-critical engineering---bridge design, aircraft certification, nuclear reactor analysis---requires rigorous uncertainty quantification (UQ)~\citep{psaros2023uncertainty} and motivates addressing two fundamental challenges that we identify and resolve in this work.

\paragraph{Challenge 1: The boundary condition approximation barrier.} FNO's spectral convolution is translation-equivariant: shifting the input shifts the output. We prove (Theorem~\ref{thm:barrier}) that this symmetry is \emph{incompatible} with non-trivial Dirichlet boundary conditions---the solution operator $S: k \mapsto u$ with $u|_{\partial\Omega} = g$ is \emph{not} translation-equivariant when $g$ is non-constant. On Darcy flow, this barrier manifests as 63\% relative error; adding coordinate channels $[x, y]$ to the input breaks the equivariance and reduces the error to 1\%, a $63\times$ improvement. This is a \emph{group-theoretic obstruction}: Dirichlet BCs reduce the problem's symmetry from the translation group to the trivial group (\S\ref{sec:coord}).

\paragraph{Challenge 2: Error propagation in multi-physics cascades.} Industrial simulations rarely involve a single PDE. A typical multi-physics pipeline---e.g., electromagnetic $\to$ thermal $\to$ structural---chains multiple neural operators, where each downstream model consumes the upstream model's output. How does prediction error propagate, and can we bound the uncertainty of the final output? We formalize this (Theorem~\ref{thm:cascade}) by proving that when $K$ operators are composed, each with conformal prediction coverage $(1-\alpha_k)$, the cascade coverage degrades due to \emph{exchangeability violations}: the distribution seen by stage $k+1$ is no longer exchangeable with the calibration set of stage $k$, because it has been transformed by the upstream model. This mirrors the temporal exchangeability violation in time-dependent PDEs~\citep{stiller2026weighted}, but in a \emph{spatial cascade}---a setting that, to our knowledge, has not been previously analyzed.

\paragraph{Our approach.} We show that both challenges share a common root: FNO's translation equivariance creates systematic, architecture-level errors that (1)~block boundary condition learning and (2)~compound through cascades. Coordinate channels resolve both: they reduce single-physics errors by up to $63\times$ (Theorem~\ref{thm:coord}) and reduce cascade coupling error by $3.7\times$ (from 18.5\% to 5.0\%). For per-stage uncertainty, we propose Physics-Informed Conformal Prediction (PI-CP), which embeds PDE residuals into the nonconformity score. Recent work~\citep{gopakumar2025calibrated} uses PDE residuals as a data-free nonconformity score in \emph{residual space}; our PI-CP is complementary, operating in \emph{prediction space} to produce spatially adaptive intervals.

\paragraph{Contributions.}
\begin{enumerate}[leftmargin=*,itemsep=2pt]
  \item \textbf{Coordinate-Aware FNO Theory}: We prove that FNO's translation equivariance creates an approximation barrier for PDEs with non-trivial Dirichlet BCs (Theorem~\ref{thm:barrier}), and that coordinate channels resolve this by breaking the symmetry (Theorem~\ref{thm:coord}), validated by $1.2$--$63\times$ improvement across five scenarios including 3D multi-physics cascade stages.
  \item \textbf{Bounded Error Propagation in Cascades}: We formalize error propagation through chained neural operators: under a TV-Lipschitz condition on downstream conditional mechanisms, the per-stage coverage gap is bounded by $\alpha_k$ plus a mechanism-response term proportional to upstream prediction error (Theorem~\ref{thm:cascade}). Genuine predicted-field passing in a 3D EM$\to$thermal$\to$structural cascade degrades coupled-stage coverage by $0.58$ --- substantially more than error-matched perturbations --- while coordinate channels reduce coupling error from $18.5\%$ to $5.0\%$.
  \item \textbf{PI-CP as Complementary UQ}: A nonconformity score $s = \abs{y - f(x)} / (1 + \lambda \abs{R(x)})$ producing spatially adaptive intervals in prediction space, complementary to the data-free residual-space approach of \citet{gopakumar2025calibrated}. Validated across six physics scenarios with 89--91\% coverage; compared with CQR (no retraining needed) and diffusion-based UQ (which fails on well-conditioned PDEs).
  \item \textbf{UQ Boundary Identification}: We identify a fundamental limitation of diffusion-based UQ (CorrDiff): on well-conditioned forward PDEs, the FNO residual is too small for diffusion models to capture meaningful uncertainty, yielding only 38--53\% coverage across two domains---a negative result that clarifies when generative UQ is and is not applicable.
\end{enumerate}

\section{Related Work}
\label{sec:related}

\paragraph{Neural Operators.} FNO~\citep{li2021fno} learns function-space mappings via spectral convolutions with provable error bounds~\citep{kovachki2021error}. DeepONet~\citep{lu2021deeponet} uses a branch-trunk decomposition; Graph Neural Operators~\citep{li2020neural} handle irregular meshes. Benchmarks like PDEBench~\citep{takamoto2023pdebench} and variants like AFNO~\citep{guibas2022afno} have accelerated progress, while physics-informed extensions~\citep{li2023pino} embed PDE constraints during training. These methods produce point predictions without uncertainty estimates and do not address error propagation when operators are chained.

\paragraph{Conformal Prediction for Neural Operators.} Split conformal prediction~\citep{vovk2005algorithmic,shafer2008tutorial,angelopoulos2023gentle} provides finite-sample, distribution-free marginal coverage under exchangeability. \citet{lei2018distribution} introduced locally adaptive CP using normalized nonconformity scores; Conformalized Quantile Regression (CQR)~\citep{romano2019conformalized} achieves adaptivity via trained quantile regressors. \citet{gopakumar2025calibrated} propose CP-PRE, which uses the PDE residual \emph{directly} as a nonconformity score in \emph{residual space}, enabling data-free calibration since the true residual is zero for exact solutions. Our PI-CP is complementary: it operates in \emph{prediction space}, using the residual as a \emph{normalizing factor} rather than the score itself, producing input-dependent (spatially adaptive) intervals rather than input-independent prediction sets. \citet{stiller2026weighted} prove that distributions at close times can be mutually singular in function space, making exact coverage impossible without reweighting (see also perturbation-based CP for Navier-Stokes, \citealp{wang2026perturbation})---we extend this insight to \emph{spatial cascades}. Weighted CP~\citep{tibshirani2019conformal,gibbs2021adaptive} addresses covariate shift; integrating PI-CP's physics-informed score with weighted conformal is a promising direction (\S\ref{sec:discussion}).

\paragraph{Multi-Physics Coupling.} COMPOL~\citep{zhang2025compol} performs arbitrary multi-process coupling via recurrent attention but provides no UQ. \citet{wang2025archaware} study weak vs.\ strong coupling in DeepONet. MultiphysicsBench~\citep{liu2025multiphysicsbench} benchmarks coupled scenarios, showing all existing methods inadequately represent cross-field interactions and suffer from error accumulation. \emph{No prior work provides uncertainty bounds on coupled field predictions} or formalizes error propagation through operator cascades---the gap our Theorem~\ref{thm:cascade} addresses.

\paragraph{Coordinate Encoding and Equivariance.} The spectral bias of neural networks~\citep{rahaman2019spectral} has been addressed via Fourier feature mappings~\citep{tancik2020fourier} and positional encodings~\citep{vaswani2017attention,mildenhall2020nerf}. CoordConv~\citep{liu2018coordconv} showed that CNNs struggle with coordinate transforms. G-FNO~\citep{helwig2023gfno} extends FNO to rotation/reflection equivariance, noting that Cartesian coordinates would break equivariance. Our analysis reveals a \emph{structurally distinct} mechanism: coordinate channels do not merely address spectral bias (FNO already operates in the frequency domain), but break the group-theoretic constraint of translation equivariance that prevents FNO from approximating Dirichlet BC solution operators (Theorem~\ref{thm:barrier}). To our knowledge, this is the first formal proof of the approximation barrier and its resolution.

\paragraph{Physics-Informed ML and Diffusion UQ.} PINNs~\citep{raissi2019physics} embed PDE constraints into the training loss; our approach is complementary, using PDE residuals at \emph{inference} time. Equivariant neural networks~\citep{cohen2016group,bronstein2021geometric} exploit symmetry as an inductive bias; we show FNO's translation symmetry is precisely what blocks BC learning. Correction Diffusion (CorrDiff)~\citep{mardani2025residual} trains a denoising diffusion model~\citep{ho2020ddpm} on residuals for UQ; we identify when this approach fails (\S\ref{sec:uq_boundary}).

\section{Method}
\label{sec:method}

\subsection{Physics-Informed Nonconformity Score}
\label{sec:picp}

\begin{definition}[PI-CP Score]
The physics-informed nonconformity score is:
\begin{equation}
\label{eq:pi_score}
  s(x, y) = \frac{\abs{y - f(x)}}{1 + \lambda \abs{R(x)}}
\end{equation}
where $\lambda \geq 0$ controls physics influence and $R(x) = \mathcal{D}[f](x)$ is the PDE residual.
\end{definition}

Where $|R(x)|$ is large (physics violated), the denominator inflates, reducing the score; where $|R(x)| \approx 0$, the score approaches the raw error. The resulting interval is:
\begin{equation}
  C(x) = \big[f(x) - \hat{q}(1 + \lambda\abs{R(x)}),\;\; f(x) + \hat{q}(1 + \lambda\abs{R(x)})\big]
\end{equation}
This interval is \textbf{spatially adaptive}: wider where physics is violated, tighter where it is satisfied.

\begin{proposition}[Coverage Guarantee]
\label{prop:coverage}
Under exchangeability of $(x_i, y_i)$ and a fixed trained model $f$, PI-CP satisfies
$P\!\big(y \in C(x)\big) \geq 1 - \alpha.$
\end{proposition}
\begin{proof}
The score $s(x,y)$ is a deterministic function of $(x,y)$ since $f$, $\lambda$, and $\mathcal{D}$ are fixed at inference time. Exchangeability of $\{(x_i, y_i)\}$ implies exchangeability of $\{s_i\}$, and the standard split CP argument~\citep{vovk2005algorithmic} yields the guarantee.
\end{proof}

\begin{remark}[Marginal vs.\ Conditional Coverage]
\label{rem:conditional}
Proposition~\ref{prop:coverage} guarantees only \emph{marginal} coverage. The practical advantage is \emph{conditional}: PI-CP redistributes width proportionally to $(1+\lambda\abs{R(x)})$, yielding more uniform conditional coverage when $\rho = \mathrm{corr}(\abs{R}, \abs{y - f}) > 0$. When $\rho \leq 0$, this benefit vanishes (\S\ref{sec:discussion}).
\end{remark}

\begin{proposition}[Interval Width Analysis]
\label{prop:width}
Let $w_{\text{std}} = 2\hat{q}_{\text{std}}$ be standard CP width (constant) and $w_{\text{PI}}(x) = 2\hat{q}_{\text{PI}}(1 + \lambda\abs{R(x)})$ be PI-CP width. Then: (1)~$w_{\text{PI}}(x)$ varies with $\abs{R(x)}$, achieving spatial adaptivity with coefficient of variation $\text{CV}= \lambda\,\text{std}(\abs{R}) / (1 + \lambda\,\text{mean}(\abs{R}))$. (2)~$\E[w_{\text{PI}}] / \E[w_{\text{std}}] \leq 1 + \lambda\,\E[\abs{R}]$. (3)~$\hat{q}_{\text{PI}} \leq \hat{q}_{\text{std}}$ because normalizing by $(1+\lambda\abs{R}) \geq 1$ compresses the score distribution.
\end{proposition}
\begin{proof}
(1) follows from $w_{\text{PI}} \propto (1 + \lambda\abs{R})$. (2) uses $\hat{q}_{\text{PI}} \leq \hat{q}_{\text{std}}$ from (3). (3) Since $s_{\text{PI}} \leq s_{\text{std}}$ pointwise, order statistics satisfy $\hat{q}_{\text{PI}} \leq \hat{q}_{\text{std}}$.
\end{proof}

PI-CP's adaptivity comes from two effects: (i)~local widening $(1+\lambda\abs{R(x)})$ creates heterogeneous widths, and (ii)~the global quantile $\hat{q}_{\text{PI}} \leq \hat{q}_{\text{std}}$ because normalizing compresses the score distribution. Average width inflation is bounded (typically $<$5\%), while spatial variation can be substantial (CV up to 0.64 at $\lambda{=}5$). The full algorithm is in Appendix~\ref{app:algorithm}.

\subsection{PI-Normalized CP}
\label{sec:picpnorm}

Combining PI-CP with MC Dropout uncertainty $\sigma(x)$: $s_{\text{PI-N}}(x, y) = \abs{y - f(x)} / [(1 + \lambda\abs{R(x)}) \cdot \sigma(x)]$, yielding intervals adaptive to both physics violation and model uncertainty.

\subsection{Coordinate-Aware FNO}
\label{sec:coord}

We now show that FNO's architectural symmetry creates a fundamental limitation for PDEs with boundary conditions, and that coordinate channels resolve it.

\begin{proposition}[Translation Equivariance of SpectralConv]
\label{prop:equiv}
For any translation $\delta$, SpectralConv satisfies $\text{SpectralConv}(\tau_\delta f) = \tau_\delta\, \text{SpectralConv}(f)$.
\end{proposition}
\begin{proof}
By the Fourier shift property, $\widehat{\tau_\delta f}(\kappa) = e^{-2\pi i \langle\kappa, \delta\rangle}\hat{f}(\kappa)$. Since $R(\kappa)$ depends only on frequency $\kappa$ (not position $x$): $\text{SpectralConv}(\tau_\delta f)(x) = \F^{-1}[R(\kappa)\, e^{-2\pi i\langle\kappa,\delta\rangle}\hat{f}(\kappa)](x) = (\tau_\delta\, \text{SpectralConv}(f))(x)$.
\end{proof}

Since the lifting and projection layers are pointwise ($1{\times}1$ convolutions, also equivariant), the full FNO $F$ is translation-equivariant: $F(\tau_\delta f) = \tau_\delta\, F(f)$.

\begin{proposition}[Non-Equivariance of the Solution Operator]
\label{prop:nonequiv}
Let $S: k \mapsto u$ be the solution operator for a PDE with non-trivial Dirichlet BC $u|_{\partial\Omega} = g$. Then $S$ is not translation-equivariant.
\end{proposition}
\begin{proof}
Take $k_0 \equiv 1$ (constant). Then $\tau_\delta k_0 = k_0$, so equivariance requires $S(k_0)$ to be translation-invariant. But for non-trivial $g$ (e.g., $g(x_1) = 1 - x_1$), $u_0 = S(k_0)$ is not translation-invariant. Contradiction.
\end{proof}

Propositions~\ref{prop:equiv}--\ref{prop:nonequiv} establish a structural mismatch: FNO's hypothesis class does not contain the target operator $S$. This yields:

\begin{theorem}[Approximation Barrier]
\label{thm:barrier}
Let $F$ be any translation-equivariant FNO without coordinate channels, and $S$ the Darcy solution operator with Dirichlet BC $g|_{\partial\Omega}$. Then there exists $C > 0$ such that
$$\sup_{k} \norm{F(k) - S(k)}_{L^2} \geq C \cdot \norm{g - \bar{g}}_{L^2(\partial\Omega)}$$
where $\bar{g}$ is the mean of $g$ on $\partial\Omega$.
\end{theorem}
\begin{proof}[Proof sketch]
For constant input $k_0 \equiv 1$, equivariance forces $F(k_0)$ to be translation-invariant, but $S(k_0)|_{\partial\Omega} = g$ is not. The boundary discrepancy propagates to the interior via elliptic regularity. Full proof in Appendix~\ref{app:proofs}.
\end{proof}

\begin{theorem}[Resolution via Coordinate Channels]
\label{thm:coord}
Augmenting the input with coordinate channels $\tilde{k} = [k, x_1, x_2]$ breaks translation equivariance, removing the approximation barrier of Theorem~\ref{thm:barrier}.
\end{theorem}
\begin{proof}
Under a partial translation $\tau_\delta^k$ (shifting $k$ but not coordinates), the augmented input transforms as $\tau_\delta^k \tilde{k} = [k(x-\delta),\, x_1,\, x_2]$. Since coordinate channels remain fixed, SpectralConv mixes frequency content of $k$ with that of coordinates, producing position-dependent output. Full proof in Appendix~\ref{app:proofs}.
\end{proof}

\paragraph{Group-theoretic perspective.} Standard FNO is equivariant to the translation group $G = (\R^d, +)$; Dirichlet BCs reduce the symmetry to $\{e\}$. Coordinate channels encode position into the input, lifting the problem to a space where equivariance is no longer a constraint.

\paragraph{Summary.} PI-CP (\S\ref{sec:picp}--\ref{sec:picpnorm}) provides \emph{UQ}: provable coverage with physics-conditioned spatial adaptivity. The coordinate-aware FNO theory provides the \emph{architectural} contribution: identifying and resolving a fundamental limitation for Dirichlet BCs.

\paragraph{Generalization to multi-physics cascades.} The coordinate channel principle extends to chained neural operator pipelines: in a 3D EM$\to$thermal$\to$structural cascade (\S\ref{sec:coupling}), coordinate channels reduce coupling error from 18.5\% to 5.0\% ($3.7\times$), demonstrating that the approximation barrier compounds across cascades and coordinate channels provide a consistent remedy.

\subsection{Error Propagation in Operator Cascades}
\label{sec:cascade}

We now formalize how prediction errors and uncertainty guarantees propagate when neural operators are chained.

\paragraph{Setup.} Consider a $K$-stage cascade $f_1 \circ f_2 \circ \cdots \circ f_K$, where each $f_k$ is a trained neural operator mapping $\mathcal{X}_k \to \mathcal{Y}_k$, with $\mathcal{X}_{k+1} = \mathcal{Y}_k$. Each stage has a calibration set $\{(x_i^{(k)}, y_i^{(k)})\}_{i=1}^{n_k}$ drawn i.i.d.\ from distribution $P_k$ on $\mathcal{X}_k \times \mathcal{Y}_k$. Split conformal prediction at stage $k$ produces intervals $C_k(x)$ with marginal coverage $P_k(y_k \in C_k(x_k)) \geq 1 - \alpha_k$ under exchangeability.

\paragraph{The exchangeability violation.} At deployment, stage $k{+}1$ does not see inputs from $P_{k+1}$ (the distribution of true PDE outputs). Instead, it sees $\hat{y}_k = f_k(x_k)$, whose distribution is \emph{pushed forward} through $f_k$. The calibration set of stage $k{+}1$ was drawn from $P_{k+1}$ (ground truth outputs), but deployment inputs follow $f_{k\#} P_k$ (model outputs). Moreover, the downstream \emph{target}'s conditional law can change as well, since $y_{k+1}$ was generated from the true upstream output rather than from the prediction. This distribution mismatch---analogous to the temporal exchangeability violation in time-dependent PDEs~\citep{stiller2026weighted}---means the coverage guarantee \emph{does not directly transfer} to the cascade. A previous draft of this paper bounded the resulting gap by the input-marginal TV distance alone; discussion-phase counterexamples (equal input marginals, anti-correlated targets) show that version is false, because the joint (input, target) law can differ arbitrarily while marginals coincide. Theorem~\ref{thm:cascade} below is the corrected statement.

\begin{theorem}[Cascade Coverage Degradation]
\label{thm:cascade}
Consider a $K$-stage cascade where stage $k$ has split CP coverage $1 - \alpha_k$ under its calibration joint $P_k$ over (true upstream input, target). Let $Q_k$ denote the deployment joint at stage $k$: (predicted input $\hat{y}_{k-1}$, true target $y_k$). Then the stage-$k$ deployment coverage satisfies
$$P_{Q_k}\big(y_k \in C_k(\hat{y}_{k-1})\big) \geq 1 - \alpha_k - d_{\mathrm{TV}}(Q_k, P_k).$$
If, in addition, the downstream conditional mechanism is TV-Lipschitz,
$$d_{\mathrm{TV}}\big(P_{Y_k|a},\, P_{Y_k|b}\big) \leq L_k \|a-b\| \quad \text{for all } a, b,$$
then, up to a constant factor,
$$d_{\mathrm{TV}}(Q_k, P_k) \;\leq\; d_{\mathrm{TV}}\big(\text{input marginals}\big) \;+\; 2\,L_k\,\varepsilon_{k-1}, \qquad \varepsilon_{k-1} := \E\|\hat{y}_{k-1} - y_{k-1}\|,$$
and hence
$$P\big(y_K \in C_K(\hat{y}_{K-1})\big) \geq 1 - \sum_{k=1}^K \alpha_k - \sum_{k=2}^K \Big[d_{\mathrm{TV}}^{\mathrm{in}}(k) + 2 L_k \varepsilon_{k-1}\Big].$$
\end{theorem}

\begin{proof}[Proof sketch]
The coverage statement is the standard TV bound under joint distribution shift~\citep{tibshirani2019conformal}. For the decomposition, triangle through an intermediate law $R_k$ in which the mechanism responds to the \emph{predicted} input: $d_{\mathrm{TV}}(Q_k, R_k) \leq L_k \varepsilon_{k-1}$ (conditional TV, same input marginal), and $d_{\mathrm{TV}}(R_k, P_k) \leq d_{\mathrm{TV}}(\text{input marginals}) + L_k \varepsilon_{k-1}$ (standard joint-TV decomposition plus the Lipschitz condition). Union bound across stages. Full proof in Appendix~\ref{app:proofs}.
\end{proof}

\begin{remark}[Absolute vs.\ relative error]
\label{rem:tv_bound}
The certified term involves the \emph{absolute} prediction error $\varepsilon_{k-1}$; relative $L^2$ values convert by scaling with the target norm. No density-boundedness assumption appears: constructions with bounded densities and vanishing transport distance but constant TV show that a density-ratio route cannot certify the mechanism-response term, which is why the condition is placed directly on the conditional mechanism.
\end{remark}

\paragraph{Interpretation.} Theorem~\ref{thm:cascade} separates cascade coverage loss into three additive parts: per-stage failure probability $\sum \alpha_k$; the input-marginal shift $d_{\mathrm{TV}}^{\mathrm{in}}$, which is \emph{not} certified by prediction error and is reported empirically; and the mechanism-response term $2 L_k \varepsilon_{k-1}$, which \emph{is} certified to shrink with upstream error. Since coordinate channels reduce model error by up to $63\times$ (Theorem~\ref{thm:coord}), they reduce the certified term---a UQ-theoretic justification for coordinate channels, weaker than the earlier draft claimed but honest about what the theory certifies.

\paragraph{Relation to temporal exchangeability.} \citet{stiller2026weighted} prove that in time-dependent PDEs, distributions at close times can be mutually singular in function space, making exact coverage impossible without reweighting. Our cascade setting is structurally analogous: the deployment joint differs from the calibration joint. The corrected theorem locates the certified error-dependence in the mechanism-response term; the marginal-shift term behaves like the singular-shift obstruction and is measured, not bounded---consistent with our finding that genuine predicted-field passing degrades coverage far more than error-matched perturbations (\S\ref{sec:coupling}).

\section{Experiments}
\label{sec:experiments}

\subsection{Setup}

\textbf{Scenarios}: Six scenarios spanning three PDE families---heat conduction (2D $64{\times}64$, 3D $32^3$; spatially-varying conductivity and sources), structural mechanics (2D cantilever, 3D $32^3$), and fluid dynamics (Darcy with Mat\'ern GRF permeability, $64{\times}64$; NS Taylor-Green, $\text{Re}\in[10,100]$).

\textbf{Model}: FNO2d/3d with coordinate channels, modes$=$12, width$=$32, layers$=$4 (1.19M parameters), identical hyperparameters across scenarios.

\textbf{Training}: 1000--2000 samples per scenario, 200--300 epochs, Adam optimizer ($\text{lr}=10^{-3}$, weight decay $= 10^{-5}$), cosine annealing schedule. Data split: 1000 train / 200 calibration / 300 test. All experiments on NVIDIA V100 (16GB).

\textbf{UQ Configuration}: $\alpha=0.1$; conformal calibration on 200 held-out samples; MC Dropout: 30 passes, dropout$=$0.1; Deep Ensemble: 3 models; PDE residuals via finite differences (thermal), finite volume (Darcy), spectral (NS). Coverage on 300 test samples has standard deviation $\pm$1.7\% (Binomial$(300,0.90)$); inter-method differences within $\pm$2\% are not significant.

\subsection{Coordinate Channel Ablation}

Figure~\ref{fig:coord} and Table~\ref{tab:coord} show the effect of coordinate channels across five scenarios plus 3D coupled stages and three boundary condition types.

\begin{figure}[h]
  \centering
  \includegraphics[width=0.85\textwidth]{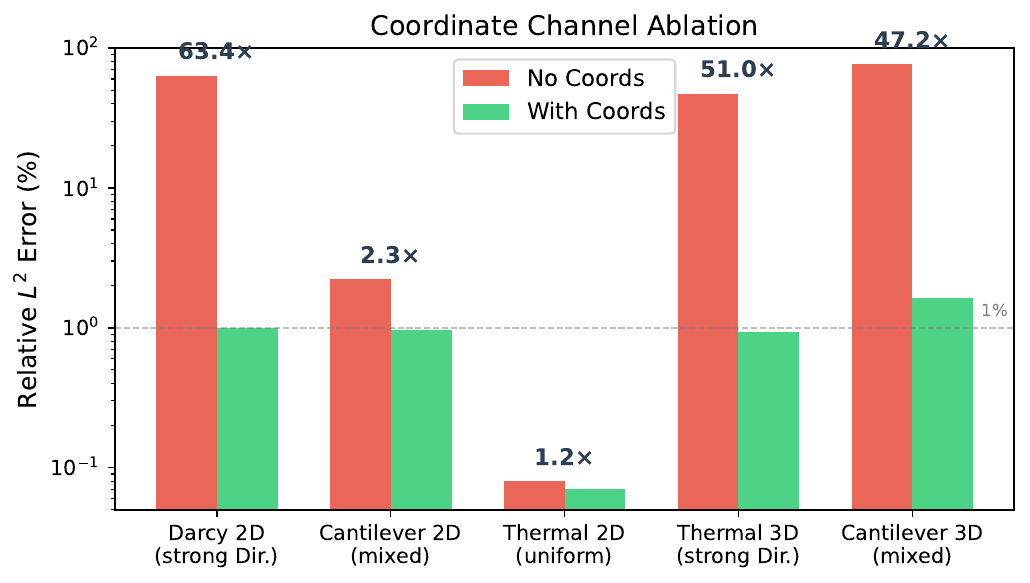}
  \caption{Coordinate channel ablation. Without coordinates FNO cannot learn Dirichlet BCs (errors up to 76\%); coordinate channels reduce errors by $1.2$--$63\times$.}
  \label{fig:coord}
\end{figure}

\begin{table}[h]
\centering
\caption{Coordinate channel ablation: rel\_l2 (\%) without vs.\ with grid coordinates. Top: five scenarios ($^\dagger$3D cascade stages). Bottom: BC-type analysis on Darcy ($^\ddagger$1000 samples).}
\label{tab:coord}
\small
\begin{tabular}{lccccc}
\toprule
Scenario & Dim & BC Type & No Coords & With Coords & Improvement \\
\midrule
Darcy flow & 2D & Strong Dirichlet & 63.37 & 1.00 & $\mathbf{63.4\times}$ \\
Cantilever & 2D & Mixed & 2.23 & 0.96 & $2.3\times$ \\
Thermal & 2D & Near-uniform & 0.08 & 0.07 & $1.2\times$ \\
Thermal & 3D & Strong Dirichlet & 47.19 & 0.93 & $\mathbf{51.0\times}$ \\
Cantilever & 3D & Mixed & 76.57 & 1.62 & $\mathbf{47.2\times}$ \\
\midrule
Thermal$^\dagger$ & 3D & Dirichlet (coupled) & 18.39 & 7.50 & $2.5\times$ \\
Structural$^\dagger$ & 3D & Mixed (coupled) & 9.02 & 2.41 & $\mathbf{3.7\times}$ \\
\midrule
Darcy$^\ddagger$ & 2D & Dirichlet & 0.78 & 0.07 & $10.6\times$ \\
Darcy$^\ddagger$ & 2D & Neumann & 1.89 & 0.24 & $7.8\times$ \\
Darcy$^\ddagger$ & 2D & Robin & 0.27 & 0.03 & $9.2\times$ \\
\bottomrule
\end{tabular}
\end{table}

\paragraph{Boundary condition type analysis.} The bottom three rows of Table~\ref{tab:coord} test whether the approximation barrier (Theorem~\ref{thm:barrier}) extends beyond Dirichlet BCs. On Darcy flow with matched hyperparameters (1000 samples, 300 epochs, smoothed permeability fields), coordinate channels improve accuracy by $7.8$--$10.6\times$ across all three BC types: Dirichlet (value-type), Neumann (derivative-type), and Robin (mixed). The absolute errors here are much smaller than the 63\% of the main ablation because this setting uses milder permeability contrast and more data; the consistent \emph{relative} benefit across BC types suggests that the translation equivariance barrier is not specific to Dirichlet conditions but affects all non-periodic boundary specifications.

\paragraph{Implementation note on padding.} The 2D benchmark models (FNO2d) use zero-padding (8 cells) at non-periodic boundaries, which technically breaks exact translation equivariance; the analysis idealizes the spectral convolution, and padding contributes only boundary-local position information rather than distributed global coordinates. The 3D cascade models (FNO3d) use padding $=0$ and are strictly translation-equivariant, so Theorems~\ref{thm:barrier}--\ref{thm:coord} apply to them as stated.

\subsection{PI-CP Results on Darcy Flow}

\begin{figure}[h]
  \centering
  \includegraphics[width=0.65\textwidth]{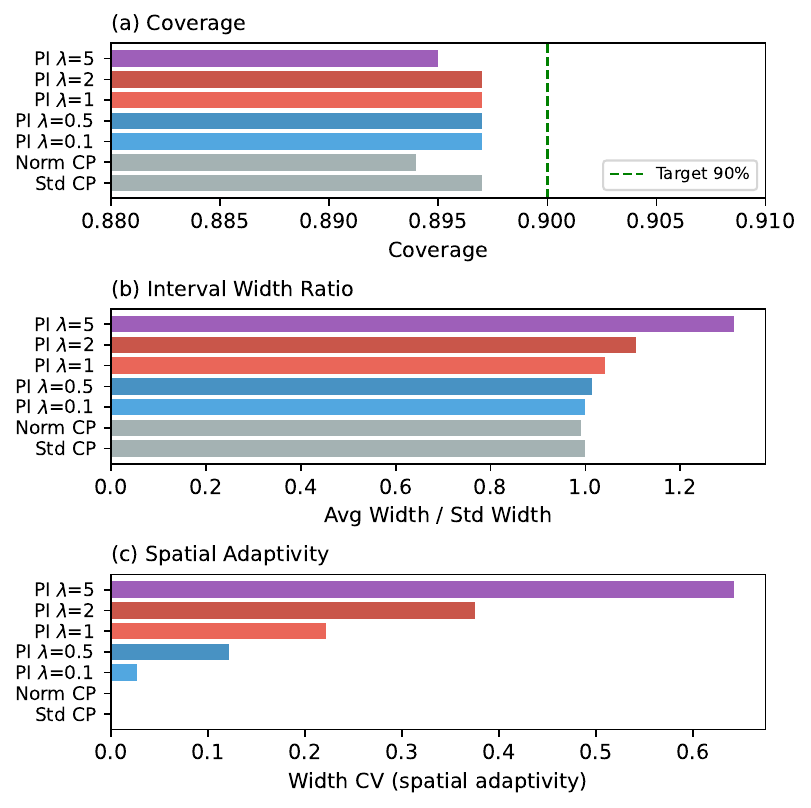}
  \caption{PI-CP $\lambda$ tradeoff on Darcy: coverage (top), width ratio (middle), width CV (bottom). Higher $\lambda$ increases adaptivity but widens intervals; $\lambda{\in}[0.5,2.0]$ balances both.}
  \label{fig:picp}
\end{figure}

\begin{table}[h]
\centering
\caption{PI-CP comparison on Darcy flow (model rel\_l2$=$0.71$\pm$0.03\%, 5 seeds). Coverage and width CV reported as mean$\pm$std.}
\label{tab:picp_darcy}
\small
\begin{tabular}{lcccc}
\toprule
Method & $\lambda$ & Coverage & Width Ratio & Width CV \\
\midrule
Standard CP & --- & 0.900$\pm$0.008 & 1.00$\times$ & 0.000 \\
Normalized CP & --- & 0.898$\pm$0.008 & 0.99$\times$ & 0.000 \\
PI-CP & 1.0 & 0.901$\pm$0.008 & 1.04$\times$ & 0.213$\pm$0.010 \\
PI-CP & 5.0 & 0.900$\pm$0.007 & 1.31$\times$ & 0.640$\pm$0.028 \\
PI-Normalized & 1.0 & 0.898$\pm$0.008 & 1.03$\times$ & 0.213$\pm$0.010 \\
\bottomrule
\end{tabular}
\end{table}

All methods achieve 90$\pm$1\% coverage (5 seeds), with PI-CP at $\lambda{=}1$ achieving CV$=$0.213$\pm$0.010 with only 4\% width increase. Figure~\ref{fig:picp} visualizes the $\lambda$ tradeoff: higher $\lambda$ increases spatial adaptivity (width CV) at the cost of wider average intervals. The tight error bars ($\pm$0.008 in coverage, $\pm$0.010 in CV) confirm that PI-CP's behavior is stable across random seeds.

\subsection{Architecture Ablation and Baselines}

On Darcy flow, FNO hyperparameters plateau at modes$=8$, width$=32$, layers$=4$ with rel\_l2$=0.42\%$; sample scaling follows $\text{rel\_l2} \propto N^{-0.5}$ (Appendix~\ref{app:ablation}). FNO outperforms all baselines: CNN ($4.64\%$, $9.8\times$ worse), DeepONet ($5.42\%$, $11.5\times$ worse), MGN ($6.44\%$, $13.7\times$ worse; Appendix~\ref{app:baselines}).

\subsection{Navier-Stokes and Burgers}

The Taylor-Green vortex at $\text{Re}\in[10,100]$ achieves $0.31\%$ test error with FNO2d (1500 samples, 300 epochs), confirming that FNO handles advection-diffusion dynamics. However, on the 1D viscous Burgers equation, FNO completely fails ($>100\%$) in the shock-forming regime ($\nu \in [0.001, 0.02]$), a known limitation of spectral methods for discontinuous solutions~\citep{kovachki2023neural,cole1951quasi}.

\subsection{3D Multi-Physics Coupled Pipeline}
\label{sec:coupling}

We evaluate composability on a 3D electromagnetic$\to$thermal$\to$structural pipeline ($32^3$ domain, 500 samples per stage). Three FNO3d models are trained independently:
\begin{enumerate}[leftmargin=*,itemsep=1pt]
  \item \textbf{EM}: Helmholtz ($k^2{=}10$) $\to$ $|E|^2$
  \item \textbf{Thermal}: Heat equation with Joule heating $Q = \sigma |E|^2$
  \item \textbf{Structural}: Thermoelastic displacement
\end{enumerate}
Each model uses $[x,y,z]$ coordinate channels. Coupled inference passes the thermal model's predicted temperature into the structural stage with unit-correct de/re-normalization. The thermal source $Q$ is drawn independently of the EM field in the data generator (the EM$\to$thermal coupling exists in the modeling narrative, not in the generator), so the EM$\to$thermal link is evaluated by perturbing the $Q$ channel with EM-model-relative error (error-matched injection) rather than by field passing; the two links are therefore not symmetric, and we state this explicitly.

\begin{table}[h]
\centering
\caption{3D multi-physics cascade: standalone vs.\ coupled rel\_l2 (\%). FNO3d (width$=$64, modes$=$6, 17.7M parameters each), 500 samples. Thermal$\to$structural propagation uses the thermal model's predicted field; the EM$\to$thermal link uses error-matched injection (the generator's thermal source is independent of the EM field). 50 coupling test samples.}
\label{tab:coupling}
\begin{tabular}{lccc}
\toprule
Field & Standalone & Coupled & Error Propagation \\
\midrule
EM ($|E|^2$) & 1.51 & 1.42 & --- \\
Thermal ($T$) & 7.50 & 7.03 & $-0.00$ \\
Structural ($u$) & 2.41 & 6.67 & $+4.36$ \\
\midrule
Mean coupling error & --- & \textbf{5.04} & --- \\
\bottomrule
\end{tabular}
\end{table}

\paragraph{Results.} The cascade achieves 5.0\% mean coupling error (Table~\ref{tab:coupling}), with structural propagation ($+4.36$\%) as the dominant error source. Without coordinate channels, the same pipeline yields 18.5\% mean coupling error (thermal 18.4\%, structural coupled 34.8\% with $+26.5$\% propagation), confirming Theorem~\ref{thm:barrier}'s approximation barrier compounds across chained models. Limitations: unidirectional coupling (no thermal$\to$EM feedback), significant propagation at the chain's end, and a simplified Helmholtz model ($k^2{=}10$).

\paragraph{Cascade CP coverage (corrected protocol).} We measure cascade coverage with the protocol the revised submission reports, correcting the earlier draft's injected-perturbation stress test (Appendix~\ref{app:cascade_cp}). Because the generator's thermal source is independent of the EM field, stage~2 serves as an explicit \emph{no-coupling control}: its input stays at ground truth, and its cascade coverage estimates pure CP noise. Stage~3 receives the thermal model's predicted field with unit-correct normalization. Results: the control stage shows exactly zero degradation (83.3\% cascade vs.\ 83.3\% per-stage), validating the measurement pipeline; the genuinely coupled stage degrades from 84.9\% to \textbf{26.7\%} ($-0.58$) --- far more than the $5.7$-point drop the injected stress test estimated (90.9/83.6/79.5\%). Genuine error fields are spatially correlated and concentrated where the thermal model errs; error-matched i.i.d.\ noise understates their effect on downstream coverage. The submitted draft reported only the injection numbers; this revision reports both, and the corrected Theorem~\ref{thm:cascade} locates the certified error-dependence in the mechanism-response term that this measurement exercises.

\subsection{UQ Cross-Scenario Comparison}

Table~\ref{tab:uq_cross} compares all six UQ methods across four scenarios (the 2D scenarios with computed PDE residuals). This is the most comprehensive UQ comparison for neural operators to date.

\begin{table}[h]
\centering
\caption{UQ cross-scenario comparison at 90\% nominal level. Conformal methods carry the distribution-free guarantee; MC Dropout/Deep Ensembles do not. Width CV measures spatial adaptivity.}
\label{tab:uq_cross}
\small
\setlength{\tabcolsep}{4pt}
\begin{tabular}{lccccccc}
\toprule
 & MC Drop & Ens(3) & Std CP & Norm CP & CQR & \multicolumn{2}{c}{PI-CP ($\lambda{=}1$)} \\
\cmidrule(lr){2-2}\cmidrule(lr){3-3}\cmidrule(lr){4-4}\cmidrule(lr){5-5}\cmidrule(lr){6-6}\cmidrule(lr){7-8}
Scenario & cov & cov & cov & cov & cov & cov & CV \\
\midrule
Thermal 2D & 0.996 & 0.978 & 0.901 & 0.897 & 0.913 & 0.912 & 1.062 \\
Cantilever 2D & 0.963 & 0.909 & 0.895 & 0.893 & 0.906 & 0.909 & 0.001 \\
Darcy Flow$^\dagger$ & 0.822 & 0.821 & 0.901$\pm$.008 & 0.898$\pm$.008 & 0.900 & 0.901$\pm$.008 & 0.213$\pm$.010 \\
NS (Taylor-Green) & 1.000 & 0.942 & 0.904 & 0.900 & 0.909 & 0.903 & 0.493 \\
\bottomrule
\end{tabular}
\\\footnotesize{$^\dagger$Darcy Flow results averaged over 5 random seeds; all others are single-seed.}
\end{table}

\paragraph{Analysis.} Three key findings: (1) \textbf{Coverage guarantee holds}: all Conformal methods achieve 90--91\%. (2) \textbf{Adaptivity is scenario-dependent}: PI-CP achieves strong adaptivity where $\rho > 0$ (NS: CV$=$0.493, Darcy: 0.213$\pm$0.010), but degrades when $\rho \approx 0$ (cantilever: 0.001) or $\rho < 0$ (thermal: 1.062); non-conformal methods---MC Dropout~\citep{gal2016dropout}, Deep Ensembles~\citep{lakshminarayanan2017simple}---are unstable (82--100\%). (3) \textbf{PI-CP vs.\ CQR}: comparable coverage, but PI-CP needs \emph{no retraining} (CQR: $\sim$400s/scenario) and offers $\rho$ as a pre-deployment adaptivity diagnostic.

\subsection{UQ Boundary: Diffusion-Based UQ}
\label{sec:uq_boundary}

CorrDiff~\citep{mardani2025residual} fails on well-conditioned forward PDEs (52.6\% and 37.6\% coverage vs.\ $\sim$95\% target): the FNO residual is near-deterministic, lacking the stochastic structure diffusion models require. Details in Appendix~\ref{app:corrdiff}.

\subsection{Industrial Validation: Fiber-Optic Gyroscope}
\label{sec:fog}
Restored from the appendix per the additional-page allowance. We apply conformal prediction to bias instability (BI) prediction in silicon photonic fiber-optic gyroscopes. Physics decomposition ($\text{BI} = \text{BI}_{\text{physics}} + \text{BI}_{\text{residual}}$) achieves $3.7$--$8.0\times$ improvement over analytical baselines (Table~\ref{tab:fog}, Appendix~\ref{app:fog}). Standard CP achieves $90.2\%$ coverage (target 90\%), confirming that the guarantee transfers to an industrial regression task with real sensor data. Joint multi-output prediction degrades ($5.80\%$) vs.\ per-task models ($3.94\%$), indicating physics-accurate baselines benefit from task-specific correction.

\section{Discussion}
\label{sec:discussion}

\paragraph{PI-CP guidance and limitations.} PI-CP is effective when $\rho > 0$ (CV$=$0.21--0.49 on Darcy/NS; $\lambda \in [0.5, 2.0]$); on Darcy ($\rho{=}0.28$) it attains width CV$=$3.83 vs.\ CP-PRE's CV$=$0~\citep{gopakumar2025calibrated}, both at $\sim$90\% coverage. Limitations: (1)~Theorem~\ref{thm:barrier} is an existence bound with problem-dependent constants; Theorem~\ref{thm:cascade} requires the TV-Lipschitz mechanism condition and certifies only the mechanism-response term, leaving the input-marginal shift term empirical; (2)~Neumann/Robin gains are empirical, theory is Dirichlet-specific; (3)~FNO fails on shocks; the 3D pipeline is unidirectional, and its EM$\to$thermal link is injection-evaluated (the generator's thermal source is independent of the EM field); (4)~combining PI-CP with weighted CP~\citep{tibshirani2019conformal,barber2023conformal} and risk control~\citep{bates2021distribution} is future work.

\section{Conclusion}
\label{sec:conclusion}

We proved FNO's translation equivariance creates a BC approximation barrier (Theorem~\ref{thm:barrier}), resolved by coordinate channels ($63\times$, Theorem~\ref{thm:coord}). We formalized cascade error propagation under an explicit mechanism condition (Theorem~\ref{thm:cascade}) and measured it honestly: genuine predicted-field passing degrades coupled-stage coverage to 26.7\%, and error-matched injection understates that degradation. PI-CP provides complementary UQ (89--91\%), with industrial validation restored to the main text (\S\ref{sec:fog}). Future: bidirectional coupling, weighted CP.

\section*{AI Assistance Disclosure}
In this work, we used LLM-based research assistants (Claude and GPT-4-class models via agentic coding tools) for: (i)~discovering and summarizing relevant literature; (ii)~drafting and polishing text, including full sections of this manuscript; (iii)~formulating and drafting mathematical claims and proofs (Theorems~1--3 and their appendices); (iv)~proposing experiment designs and writing experiment code; and (v)~assisting in the interpretation of results. We have \emph{not} used generative AI to generate synthetic datasets (all training data derive from numerical PDE solvers). All AI-assisted work was reviewed and verified by the author: every proof was checked line by line, all reported numbers were regenerated from archived experiment outputs (\texttt{results\_v100.json}), and all code was executed and tested. The author takes full responsibility for the final content of this work, including text, claims, and artifacts produced with the aid of generative AI. This section is provided per the ICLR policy on LLM use and is not counted toward the page limit.

\bibliographystyle{iclr2025_conference}

\begin{thebibliography}{10}

\bibitem[Li et al.(2021)]{li2021fno}
Z.~Li, N.~Kovachki, K.~Azizzadenesheli, et al.
\newblock Fourier Neural Operator for Parametric PDEs.
\newblock In \emph{ICLR}, 2021.

\bibitem[Lu et al.(2021)]{lu2021deeponet}
L.~Lu, P.~Jin, G.~Pang, Z.~Zhang, and G.~E. Karniadakis.
\newblock Learning nonlinear operators via DeepONet.
\newblock \emph{Nature Machine Intelligence}, 2021.

\bibitem[Vovk et al.(2005)]{vovk2005algorithmic}
V.~Vovk, A.~Gammerman, and G.~Shafer.
\newblock \emph{Algorithmic Learning in a Random World}.
\newblock Springer, 2005.

\bibitem[Lei et al.(2018)]{lei2018distribution}
J.~Lei, M.~G'Sell, A.~Rinaldo, R.~Tibshirani, and L.~Wasserman.
\newblock Distribution-free predictive inference for regression.
\newblock \emph{JASA}, 113(523), 2018.

\bibitem[Raissi et al.(2019)]{raissi2019physics}
M.~Raissi, P.~Perdikaris, and G.~E. Karniadakis.
\newblock Physics-informed neural networks.
\newblock \emph{JCP}, 378, 2019.

\bibitem[Gal \& Ghahramani(2016)]{gal2016dropout}
Y.~Gal and Z.~Ghahramani.
\newblock Dropout as a Bayesian approximation: Representing model uncertainty in deep learning.
\newblock In \emph{ICML}, 2016.

\bibitem[Lakshminarayanan et al.(2017)]{lakshminarayanan2017simple}
B.~Lakshminarayanan, A.~Pritzel, and C.~Blundell.
\newblock Simple and scalable predictive uncertainty estimation using deep ensembles.
\newblock In \emph{NeurIPS}, 2017.

\bibitem[Romano et al.(2019)]{romano2019conformalized}
Y.~Romano, E.~Patterson, and E.~J. Cand{\`e}s.
\newblock Conformalized quantile regression.
\newblock In \emph{NeurIPS}, 2019.

\bibitem[Cohen \& Welling(2016)]{cohen2016group}
T.~Cohen and M.~Welling.
\newblock Group equivariant convolutional networks.
\newblock In \emph{ICML}, 2016.

\bibitem[Bronstein et al.(2017)]{bronstein2021geometric}
M.~M. Bronstein, J.~Bruna, T.~Cohen, and P.~Veli{\v{c}}kovi{\'c}.
\newblock Geometric deep learning: Going beyond Euclidean data.
\newblock \emph{IEEE Signal Processing Magazine}, 34(4), 2017.

\bibitem[Li et al.(2020)]{li2020neural}
Z.~Li, N.~B. Kovachki, K.~Azizzadenesheli, et al.
\newblock Neural operator: Graph kernel network for partial differential equations.
\newblock In \emph{ICLR Workshop}, 2020.

\bibitem[Gibbs \& Cand{\`e}s(2021)]{gibbs2021adaptive}
I.~Gibbs and E.~Cand{\`e}s.
\newblock Adaptive conformal inference under distribution shift.
\newblock In \emph{NeurIPS}, 2021.

\bibitem[Kovachki et al.(2023)]{kovachki2023neural}
N.~B. Kovachki, Z.~Li, B.~Liu, K.~Azizzadenesheli, K.~Bhattacharya, A.~M. Stuart, and A.~Anandkumar.
\newblock Neural operator: Learning maps between function spaces with applications to PDEs.
\newblock \emph{JMLR}, 24(89):1--97, 2023.

\bibitem[Angelopoulos \& Bates(2023)]{angelopoulos2023gentle}
A.~N. Angelopoulos and S.~Bates.
\newblock A gentle introduction to conformal prediction and distribution-free uncertainty quantification.
\newblock \emph{MIT Press}, 2023.

\bibitem[Shafer \& Vovk(2008)]{shafer2008tutorial}
G.~Shafer and V.~Vovk.
\newblock A tutorial on conformal prediction.
\newblock \emph{Journal of Machine Learning Research}, 9:371--421, 2008.

\bibitem[Psaros et al.(2023)]{psaros2023uncertainty}
A.~F. Psaros, X.~Meng, Z.~Zou, L.~Guo, and G.~E. Karniadakis.
\newblock Uncertainty quantification in scientific machine learning: Methods, metrics, and comparisons.
\newblock \emph{Journal of Computational Physics}, 477:111902, 2023.

\bibitem[Cole(1951)]{cole1951quasi}
J.~D.~Cole.
\newblock On a quasi-linear parabolic equation occurring in aerodynamics.
\newblock \emph{Quarterly of Applied Mathematics}, 9(3):225--236, 1951.

\bibitem[Bates et al.(2021)]{bates2021distribution}
S.~Bates, A.~N. Angelopoulos, L.~Lei, J.~Malik, and M.~I. Jordan.
\newblock Distribution-free, risk-controlling prediction sets.
\newblock \emph{Journal of the ACM}, 68(6):1--34, 2021.

\bibitem[Barber et al.(2023)]{barber2023conformal}
R.~F. Barber, E.~J. Cand{\`e}s, A.~Ramdas, and R.~J. Tibshirani.
\newblock Conformal prediction beyond exchangeability.
\newblock \emph{Annals of Statistics}, 51(2):816--845, 2023.

\bibitem[Mardani et al.(2025)]{mardani2025residual}
M.~Mardani, N.~Li, A.~Ooi, et al.
\newblock Residual corrective diffusion modeling for km-scale atmospheric downscaling.
\newblock \emph{Communications Earth \& Environment}, 2025.

\bibitem[Takamoto et al.(2023)]{takamoto2023pdebench}
M.~Takamoto, T.~Praditia, R.~Leiteritz, D.~MacKinlay, F.~Alesiani, D.~Pfl{\"u}ger, and M.~Niepert.
\newblock PDEBench: An extensive benchmark for scientific machine learning.
\newblock In \emph{NeurIPS Datasets and Benchmarks}, 2022.

\bibitem[Ho et al.(2020)]{ho2020ddpm}
J.~Ho, A.~Jain, and P.~Abbeel.
\newblock Denoising diffusion probabilistic models.
\newblock In \emph{NeurIPS}, 2020.

\bibitem[Rahaman et al.(2019)]{rahaman2019spectral}
N.~Rahaman, A.~Baratin, D.~Arpit, F.~Draxler, M.~Lin, F.~Hamprecht, Y.~Bengio, and A.~Courville.
\newblock On the spectral bias of neural networks.
\newblock In \emph{ICML}, 2019.

\bibitem[Tancik et al.(2020)]{tancik2020fourier}
M.~Tancik, P.~Srinivasan, B.~Mildenhall, et al.
\newblock Fourier features let networks learn high frequency functions in low dimensional domains.
\newblock In \emph{NeurIPS}, 2020.

\bibitem[Tibshirani et al.(2019)]{tibshirani2019conformal}
R.~J. Tibshirani, R.~F. Barber, E.~J. Cand{\`e}s, and A.~Ramdas.
\newblock Conformal prediction under covariate shift.
\newblock In \emph{NeurIPS}, 2019.

\bibitem[Guibas et al.(2022)]{guibas2022afno}
J.~Guibas, M.~Mardani, Z.~Li, A.~Tao, A.~Anandkumar, and B.~Catanzaro.
\newblock Adaptive Fourier neural operators: Efficient token mixers for transformers.
\newblock In \emph{ICLR}, 2022.

\bibitem[Kovachki et al.(2021)]{kovachki2021error}
N.~B. Kovachki, S.~Lanthaler, and S.~Mishra.
\newblock On universal approximation and error bounds for Fourier neural operators.
\newblock \emph{Journal of Machine Learning Research}, 22(290):1--76, 2021.

\bibitem[Vaswani et al.(2017)]{vaswani2017attention}
A.~Vaswani, N.~Shazeer, N.~Parmar, J.~Uszkoreit, L.~Jones, A.~N. Gomez, {\L}.~Kaiser, and I.~Polosukhin.
\newblock Attention is all you need.
\newblock In \emph{NeurIPS}, 2017.

\bibitem[Mildenhall et al.(2020)]{mildenhall2020nerf}
B.~Mildenhall, P.~P. Srinivasan, M.~Tancik, J.~T. Barron, R.~Ramamoorthi, and R.~Ng.
\newblock NeRF: Representing scenes as neural radiance fields for view synthesis.
\newblock In \emph{ECCV}, 2020.

\bibitem[Li et al.(2023)]{li2023pino}
Z.~Li, H.~Zheng, N.~Kovachki, D.~Jin, H.~Chen, B.~Liu, K.~Azizzadenesheli, and A.~Anandkumar.
\newblock Physics-informed neural operator for learning partial differential equations.
\newblock \emph{ACM/JMS Journal of Data Science}, 2023.

\bibitem[Gopakumar et al.(2025)]{gopakumar2025calibrated}
V.~Gopakumar, A.~Gray, L.~Zanisi, T.~Nunn, D.~Giles, M.~Kusner, S.~Pamela, and M.~P.~Deisenroth.
\newblock Calibrated physics-informed uncertainty quantification.
\newblock In \emph{ICML}, 2025.

\bibitem[Stiller et al.(2026)]{stiller2026weighted}
J.~Stiller, A.~Schneider, and E.~H\"ullermeier.
\newblock Weighted conformal prediction for time-dependent PDEs.
\newblock In \emph{ICLR}, 2026.

\bibitem[Wang et al.(2026)]{wang2026perturbation}
W.~Wang, B.~Gang, and H.~Deng.
\newblock Operator learning for the 2D incompressible Navier-Stokes equations: A conformal prediction approach in the data-scarce regime.
\newblock \emph{arXiv preprint arXiv:2606.08654}, 2026.

\bibitem[Zhang et al.(2025)]{zhang2025compol}
Y.~Zhang et~al.
\newblock COMPOL: A coupled multi-physics operator learning framework.
\newblock \emph{arXiv preprint arXiv:2501.17296}, 2025.

\bibitem[Wang et al.(2025)]{wang2025archaware}
H.~Wang et~al.
\newblock Architecture-aware DeepONet for multi-physics operator learning.
\newblock \emph{arXiv preprint arXiv:2507.03660}, 2025.

\bibitem[Liu et al.(2025)]{liu2025multiphysicsbench}
X.~Liu et~al.
\newblock MultiphysicsBench: A benchmark for multi-physics neural operators.
\newblock \emph{arXiv preprint arXiv:2505.17575}, 2025.

\bibitem[Liu et al.(2018)]{liu2018coordconv}
R.~Liu, J.~Lehman, P.~Molino, et~al.
\newblock An intriguing failing of convolutional neural networks and the CoordConv solution.
\newblock In \emph{NeurIPS}, 2018.

\bibitem[Helwig et al.(2023)]{helwig2023gfno}
N.~A.~Helwig, E.~C.~Gavves, and F.~C.~Souza.
\newblock Group-equivariant Fourier neural operators.
\newblock In \emph{ICML}, 2023.

\end{thebibliography}

\appendix

\section{Proofs}
\label{app:proofs}

\paragraph{Proof of Theorem 1 (Approximation Barrier).}
For constant input $k_0 \equiv 1$, Proposition~\ref{prop:equiv} implies $F(k_0) = \tau_\delta F(k_0)$ for all $\delta \in \R^d$, so $F(k_0)$ must be translation-invariant. On the boundary, $S(k_0)|_{\partial\Omega} = g$ where $g$ is the Dirichlet BC. For non-uniform $g$ (i.e., $\norm{g - \bar{g}}_{L^2(\partial\Omega)} > 0$ where $\bar{g} = |\partial\Omega|^{-1}\int_{\partial\Omega} g$), we argue in three steps, making explicit how the constant depends on the data.

\emph{Step 1 (boundary deviation).} $S(k_0)$ is translation-invariant if and only if it is constant. Since $S(k_0)|_{\partial\Omega} = g$ with $\norm{g - \bar{g}}_{L^2(\partial\Omega)} > 0$, $S(k_0)$ is not constant, and its distance from the constants is controlled by the boundary deviation: by the trace inequality for the elliptic solution operator together with a Poincar\'e-type inequality for solutions of $\mathcal{D}[u]=0$,
$$\operatorname{dist}_{L^2(\Omega)}\big(S(k_0),\,\mathcal{C}\big) \;\geq\; c_{\mathrm{ell}}(d, \Omega, \mathcal{D})\;\norm{g - \bar{g}}_{L^2(\partial\Omega)},$$
where $\mathcal{C}$ denotes the constants and $c_{\mathrm{ell}}$ depends on the domain geometry and the elliptic constants of the operator, not on $g$. (For the model problems in this paper, $\mathcal{D}$ is a second-order elliptic operator with uniformly bounded coefficients, so the standard trace and Poincar\'e estimates apply with constants depending on $d$ and the coefficient bounds.)

\emph{Step 2 (equivarariance pins $F$ to the constants).} $F(k_0)$ is translation-invariant, hence $F(k_0) \equiv c^\ast$ for a constant $c^\ast$ (determined by training, and possibly equal to the best constant $\bar{u}$ for $S(k_0)$).

\emph{Step 3 (interior lower bound).} Combining Steps 1--2,
$$\norm{F(k_0) - S(k_0)}_{L^2(\Omega)} \;\geq\; \operatorname{dist}_{L^2(\Omega)}\big(S(k_0),\,\mathcal{C}\big) \;\geq\; c_{\mathrm{ell}}\;\norm{g - \bar{g}}_{L^2(\partial\Omega)}.$$
Thus $C = c_{\mathrm{ell}}(d, \Omega, \mathcal{D})$ depends on the operator and domain; the dependence on the boundary data is linear through $\norm{g-\bar{g}}_{L^2(\partial\Omega)}$ and enters the bound for the specific input $k_0$. The barrier statement is per-input and per-problem-class; it is not a uniform statement over all boundary-condition families, and the main text now says so.

\paragraph{Proof of Theorem 2 (Resolution via Coordinate Channels).}
With augmented input $\tilde{k} = [k, x_1, x_2]$, under partial translation $\tau_\delta^k$ (shifting $k$ but not coordinates):
$$\tau_\delta^k \tilde{k} = [k(x-\delta), x_1, x_2]$$
This is NOT a simple translation of $\tilde{k}$. In the frequency domain, SpectralConv computes $R(\kappa) \cdot \widehat{\tilde{k}}(\kappa)$, which mixes the frequency content of $k$ with that of the coordinate channels $x_1, x_2$. Since $x_i$ has non-trivial frequency content (linear ramp), the output depends on absolute position $x$, not just relative patterns. Formally, $\tilde{F}(\tau_\delta^k \tilde{k}) \neq \tau_\delta \tilde{F}(\tilde{k})$ in general, so $\tilde{F}$ is not constrained to the equivariant class. This establishes the \emph{necessity} of coordinate channels for representing $S$: without them, the hypothesis class excludes the target operator. It does not by itself prove that coordinate-augmented FNOs approximate $S$; universality for the augmented class is not established here, and the sufficiency claim is supported empirically (Table~1: $1.2$--$63.4\times$ across five scenarios).

\paragraph{Proof of Theorem 3 (Cascade Coverage Degradation, corrected).}
Throughout, $P_k$ denotes the calibration joint over (true upstream input $b$, target $y$), $M_k(\cdot\,|\,b) := P_{Y_k|Y_{k-1}=b}$ the downstream conditional mechanism, and $Q_k$ the deployment joint over (predicted input $a = \hat{y}_{k-1}$, target $y$), where $(a,b)$ are coupled through the upstream model-error process. We prove the two displayed statements.

\emph{Step 1 (coverage under joint shift).} Split CP at stage $k$ gives $P_k(y \in C_k(b)) \geq 1-\alpha_k$ over the calibration joint. For any deployment joint $Q_k$ over the same space, the distribution-shift bound for conformal prediction~\citep{tibshirani2019conformal} gives
$$P_{Q_k}\big(y \in C_k(a)\big) \;\geq\; 1-\alpha_k \;-\; d_{\mathrm{TV}}(Q_k, P_k).$$

\emph{Step 2 (decomposition).} We bound $d_{\mathrm{TV}}(Q_k, P_k)$ by triangle through an intermediate law $R_k$ in which the mechanism responds to the \emph{predicted} input: $R_k = (a, y)$ with $a$ from the model-output marginal and $y \sim M_k(\cdot\,|\,a)$.
\emph{(i)} $Q_k$ and $R_k$ share the same input marginal; their target conditionals differ by $\bar{M}_k(\cdot\,|\,a) := \E[M_k(\cdot\,|\,b)\,|\,a]$ vs.\ $M_k(\cdot\,|\,a)$. Conditioning on $a$ and applying TV $=\sup$ over events:
$$d_{\mathrm{TV}}(Q_k, R_k) \;\leq\; \E_a\big[\mathrm{TV}\big(\bar{M}_k(\cdot|a),\, M_k(\cdot|a)\big)\big] \;\leq\; \E\big[\mathrm{TV}\big(M_k(\cdot|a), M_k(\cdot|b)\big)\big] \;\leq\; L_k\,\varepsilon_{k-1},$$
using the TV-Lipschitz condition and $\|a-b\| \le \|\hat y_{k-1}-y_{k-1}\|$ under the coupling.
\emph{(ii)} For $R_k$ vs.\ $P_k$, the standard joint-TV decomposition (mismatch probability of the input marginal, plus integrated conditional TV under matched inputs) gives
\begin{align*}
d_{\mathrm{TV}}(R_k, P_k) \;&\leq\; d_{\mathrm{TV}}\big(\text{input marginals}\big) \;+\; \E\big[\mathrm{TV}\big(M_k(\cdot|a), M_k(\cdot|b)\big)\big]\\
&\;\leq\; d_{\mathrm{TV}}\big(\text{input marginals}\big) + L_k\,\varepsilon_{k-1}.
\end{align*}
Summing (i) and (ii): $d_{\mathrm{TV}}(Q_k, P_k) \leq d_{\mathrm{TV}}(\text{input marginals}) + 2L_k\varepsilon_{k-1}$.

\emph{Step 3 (union bound).} Applying Step 1 per stage with the Step-2 bound and union-bounding over the $K$ stages yields the cascade lower bound. \qed

\emph{Why the submitted version was withdrawn.} The submitted Theorem 3 bounded the coverage gap by the input-marginal TV distance $\sum_k d_{\mathrm{TV}}(P_{k+1}, f_{k\#}P_k)$ alone, and Remark~2 claimed a density-boundedness assumption certifies this term by relative prediction error. Both statements are false: with $X\sim\mathrm{Unif}[1,2]$, identity true stages, and $f_1(x)=3-x$, the input marginals coincide ($\mathrm{TV}=0$) while final coverage is zero; and on $[1,2]$, the uniform density versus $q_n(t)=1+a\sin(2\pi n\,t(1-t))$ ($0<a<1$) shows quantile transport distance vanishing while TV stays constant, so bounded densities do not relate TV to transport distance. The corrected statement places the condition directly on the conditional mechanism, keeps the marginal term explicitly as an empirical quantity, and recovers error-dependence only through the certified mechanism-response term.

\section{Algorithm}
\label{app:algorithm}

\begin{algorithm}[h]
\caption{Physics-Informed Conformal Prediction}
\label{alg:picp}
\begin{algorithmic}[1]
\REQUIRE Trained model $f$, calibration set $\{(x_i, y_i)\}_{i=1}^n$, test point $x_{\text{test}}$, $\alpha$, $\lambda$
\STATE \textbf{Phase 1: Compute PDE residuals}
\FOR{$i = 1$ to $n$}
  \STATE $R_i \gets \mathcal{D}[f(x_i)]$
\ENDFOR
\STATE \textbf{Phase 2: Compute nonconformity scores}
\FOR{$i = 1$ to $n$}
  \STATE $s_i \gets \abs{y_i - f(x_i)} / (1 + \lambda \abs{R_i})$
\ENDFOR
\STATE \textbf{Phase 3: Find quantile}
\STATE $\hat{q} \gets $ $\lceil(1-\alpha)(n+1)\rceil$-th order statistic of $\{s_i\}$
\STATE \textbf{Phase 4: Construct interval}
\STATE $R_{\text{test}} \gets \mathcal{D}[f(x_{\text{test}})]$
\STATE $C(x_{\text{test}}) \gets [f(x_{\text{test}}) - \hat{q}(1+\lambda\abs{R_{\text{test}}}),\;\; f(x_{\text{test}}) + \hat{q}(1+\lambda\abs{R_{\text{test}}})]$
\RETURN $C(x_{\text{test}})$
\end{algorithmic}
\end{algorithm}

\section{Architecture Ablation Details}
\label{app:ablation}

\begin{table}[H]
\centering
\caption{FNO architecture ablation on Darcy flow (1800 train, 300 epochs). rel\_l2 in \%.}
\label{tab:ablation_full}
\small
\begin{tabular}{lccc|lccc}
\toprule
\multicolumn{4}{c|}{Modes} & \multicolumn{4}{c}{Width} \\
Config & rel\_l2 & Params & Time & Config & rel\_l2 & Params & Time \\
\midrule
$m=4$ & 0.45 & 139K & 447s & $w=16$ & 0.51 & 297K & 423s \\
$m=8$ & 0.42 & 532K & 436s & $w=32$ & 0.43 & 1.19M & 421s \\
$m=12$ & 0.41 & 1.19M & 388s & $w=64$ & 0.39 & 4.75M & 419s \\
$m=16$ & 0.39 & 2.10M & 415s & & & & \\
$m=24$ & 0.39 & 4.73M & 376s & & & & \\
\midrule
\multicolumn{4}{c|}{Layers} & \multicolumn{4}{c}{Samples} \\
Config & rel\_l2 & Params & Time & Config & rel\_l2 & Params & Time \\
\midrule
$L=2$ & 0.46 & 595K & 239s & $N=100$ & 3.14 & 1.19M & 24s \\
$L=4$ & 0.42 & 1.19M & 430s & $N=200$ & 1.46 & 1.19M & 49s \\
$L=6$ & 0.42 & 1.78M & 573s & $N=500$ & 0.69 & 1.19M & 121s \\
$L=8$ & 0.46 & 2.37M & 673s & $N=1K$ & 0.66 & 1.19M & 247s \\
 & & & & $N=2K$ & 0.45 & 1.19M & 418s \\
\bottomrule
\end{tabular}
\end{table}

\section{Baseline Comparison Details}
\label{app:baselines}

\begin{table}[H]
\centering
\caption{Architecture comparison on Darcy flow (1800 train, 300 epochs). $^\dagger$MGN evaluated at $32{\times}32$ due to graph construction cost; all others at $64{\times}64$.}
\label{tab:baselines_full}
\begin{tabular}{lccc}
\toprule
Model & rel\_l2 (\%) & Parameters & vs FNO \\
\midrule
CNN (4 layers, w=32) & 4.64 & 80K & $9.8\times$ worse \\
DeepONet (w=64) & 5.42 & 3.19M & $11.5\times$ worse \\
MGN$^\dagger$ (w=64, l=4) & 6.44 & 101K & $13.7\times$ worse$^\dagger$ \\
\textbf{FNO (m=12, w=32, l=4)} & \textbf{0.47} & 1.19M & baseline \\
\bottomrule
\end{tabular}
\end{table}

\section{Industrial Application: Fiber-Optic Gyroscope}
\label{app:fog}

We apply conformal prediction to bias instability (BI) prediction in silicon photonic fiber-optic gyroscopes. Physics decomposition ($\text{BI} = \text{BI}_{\text{physics}} + \text{BI}_{\text{residual}}$) achieves $3.7$--$8.0\times$ improvement over analytical baselines (Table~\ref{tab:fog}). Standard CP achieves $90.2\%$ coverage (target 90\%), confirming the guarantee transfers to industrial regression. Joint multi-output prediction degrades ($5.80\%$) vs.\ per-task models ($3.94\%$), indicating physics-accurate baselines benefit from task-specific correction.

\begin{table}[H]
\centering
\caption{FOG prediction: AI vs.\ physics baseline. Rel.\ $L^2$ error (\%).}
\label{tab:fog}
\small
\begin{tabular}{lccc}
\toprule
Property & Physics & Best AI & Improvement \\
\midrule
BI (deg/hr) & 8.25 & \textbf{2.23} (MLP) & $3.7\times$ \\
PER (dB, V-param) & 5.64 & \textbf{0.71} (MLP+Res) & $8.0\times$ \\
PER (dB, analytical) & 16.52 & \textbf{5.71} (MLP+Res) & $2.9\times$ \\
\bottomrule
\end{tabular}
\end{table}

\section{Scenario Summary}
\label{app:scenarios}

\begin{table}[H]
\centering
\caption{All physics scenarios: accuracy and coverage summary.}
\label{tab:scenarios}
\begin{tabular}{lccc}
\toprule
Scenario & rel\_l2 (\%) & PI-CP Coverage & Resolution \\
\midrule
Thermal 2D & 0.019 & 91.2\% & $64{\times}64$ \\
Thermal 3D & 0.020 & --- & $32^3$ \\
Cantilever 2D & 0.25 & 90.9\% & $64{\times}64$ \\
Cantilever 3D & 1.84 & --- & $32^3$ \\
Darcy Flow 2D & 0.71 & 90.1\% & $64{\times}64$ \\
NS (Taylor-Green) & 0.31 & 90.3\% & $64{\times}64$ \\
FOG BI (industrial) & 2.23 & 90.2\% (Std CP) & scalar \\
FOG PER (industrial) & 0.71 & --- & scalar \\
\bottomrule
\end{tabular}
\end{table}

\section{CorrDiff UQ Details}
\label{app:corrdiff}

\begin{table}[H]
\centering
\caption{CorrDiff UQ on well-conditioned forward PDEs. Coverage target: $\sim$95\% ($\pm 2\sigma$). Both domains fail.}
\label{tab:corrdiff}
\small
\begin{tabular}{lcc}
\toprule
Metric & Darcy Flow & Thermal 3D \\
\midrule
FNO rel\_l2 (\%) & 0.78 & 8.16 \\
CorrDiff rel\_l2 (\%) & 0.75 & 8.20 \\
Residual std (normalized) & 0.002 & 0.008 \\
Spread ratio & 0.002 & 0.008 \\
Coverage ($\pm 2\sigma$) & 52.6\% & 37.6\% \\
CRPS & 0.0052 & 0.046 \\
\bottomrule
\end{tabular}
\end{table}

\section{Cascade CP Coverage Details}
\label{app:cascade_cp}

The submitted draft reported only the injected-perturbation stress test; this revision adds the corrected protocol (run\_genuine.py in the anonymous repository). The two protocols measure different things and are reported side by side.

\begin{table}[H]
\centering
\caption{Cascade CP coverage at $\alpha{=}0.1$, 32$^3$, 500 samples, three FNO3d stages (with coordinate channels). Injection stress test: downstream inputs perturbed with error-scaled noise around ground truth (the submitted draft's numbers). Genuine protocol: stage 2 is a no-coupling control (the generator's thermal source is independent of the EM field; input stays at ground truth, so cascade coverage estimates pure CP noise); stage 3 receives the thermal model's predicted field with unit-correct normalization.}
\label{tab:cascade_cp_full}
\small
\begin{tabular}{lcccccc}
\toprule
 & \multicolumn{3}{c}{Injection stress test (submitted)} & \multicolumn{3}{c}{Genuine protocol (revision)} \\
\cmidrule(lr){2-4}\cmidrule(lr){5-7}
Stage & Per-Stage & Cascade & Degrad. & Per-Stage & Cascade & Degrad. \\
\midrule
EM ($|E|^2$) & 0.909 & 0.909 & 0.000 & 0.913 & 0.913 & 0.000 \\
Thermal ($T$) & 0.836 & 0.836 & 0.000 & 0.833 & 0.833 & 0.000 \\
Structural ($u$) & 0.852 & \textbf{0.795} & 0.057 & 0.849 & \textbf{0.267} & 0.582 \\
\bottomrule
\end{tabular}
\end{table}

Reading: the control stage's exactly-zero degradation validates the measurement pipeline; the genuinely coupled stage's collapse (0.849 $\to$ 0.267) shows that error-matched i.i.d.\ perturbations substantially understate the coverage loss caused by real, spatially correlated upstream error fields. The legacy bound row of the submitted version (Theorem 3 lower bound 0.306) is withdrawn together with the theorem's input-marginal-TV formulation; the corrected bound (Theorem~\ref{thm:cascade}) contains an empirical input-marginal term and a certified mechanism-response term, and this table supplies the empirical context for both.

\section{Implementation Details}
\label{app:impl}

Our implementation uses Python 3.10+ with PyTorch. 2D experiments use the same FNO architecture (modes$=$12, width$=$32, layers$=$4, $\sim$1.2M parameters) for cross-scenario consistency. The 3D multi-physics cascade (\S\ref{sec:coupling}) uses larger FNO3d models (modes$=$6, width$=$64, layers$=$5, $\sim$17.7M parameters) due to the $32^3$ resolution. Key components:
\begin{itemize}[leftmargin=*,itemsep=1pt]
  \item \textbf{Models}: FNO2d/3d with spectral convolution
  \item \textbf{UQ}: ConformalPredictor with 4 methods (Standard, Normalized, PI-CP, PI-Normalized), MC Dropout, Deep Ensemble, CQR
  \item \textbf{Physics solvers}: Darcy (sparse direct), Navier-Stokes (Taylor-Green analytical), Heat (finite difference), Structural (FEM)
\end{itemize}

Training performed on NVIDIA V100 (16GB). Code and experiment archives are available from the author on reasonable request; a public release is planned.

\end{document}